\documentclass[11pt]{article}
\usepackage[final]{acl}

\usepackage{times}
\usepackage{latexsym}
\usepackage[T1]{fontenc}
\usepackage[utf8]{inputenc}
\usepackage{microtype}
\usepackage{amsfonts}
\usepackage{float}
\usepackage{url}
\usepackage{booktabs}
\usepackage{multicol}
\usepackage{multirow}
\usepackage{graphicx}
\usepackage{makecell}
\usepackage{algorithm}
\usepackage{algpseudocode}
\usepackage{amsmath}
\usepackage{wrapfig}
\usepackage{lineno}
\usepackage{pifont}
\usepackage{cleveref}
\usepackage{colortbl}
\usepackage[table]{xcolor}
\PassOptionsToPackage{table}{xcolor}
\usepackage{tcolorbox}
\newcommand{\nn}{\notag}

\usepackage{mathtools}

\usepackage[mathscr]{euscript}
\usepackage{bm}
\usepackage{xspace}

\newcommand{\Vc}{\mathcal{V}}

\newcommand{\x}{\vct{x}}

\newcommand{\w}{{\vct{w}}}

\newcommand{\e}{\vct{e}}

\newcommand{\y}{\vct{y}}

\newcommand{\z}{\vct{z}}

\newcommand{\hb}{\vct{h}}

\newcommand{\W}{\mtx{W}}

\newcommand{\vct}[1]{\bm{#1}}
\newcommand{\mtx}[1]{\bm{#1}}

\usepackage{subcaption}
\usepackage{microtype}
\usepackage{algorithm}
\usepackage{multirow}
\newcommand{\ours}{\texttt{For-Value}}
\newcommand{\dfv}{\texttt{DataInf}}
\newcommand{\hfree}{\texttt{Hessian-free}}
\newcommand{\emb}{\texttt{Emb}}
\newcommand{\hyp}{\texttt{HyperINF}}
\usepackage{algpseudocode}
\newtheorem{theorem}{Theorem}

\newtheorem{assumption}{Assumption}
\newtheorem{definition}{Definition}

\crefname{figure}{Fig.}{Fig.}
\crefname{table}{Tab.}{Tab.}
\crefname{equation}{Eq.}{Eq.}
\crefname{section}{Sec.}{Sec.}
\crefname{appendix}{App.}{App.}

\usepackage{inconsolata}

\usepackage{graphicx}

\title{For-Value: Efficient Forward-Only Data Valuation for finetuning LLMs and VLMs}

\author{Wenlong Deng$^{1\star}$, Qi Zeng$^{3}$,  Jiaming Zhang$^{1}$, Minghui Chen$^{1}$, Zixin Ding$^{3}$,\\ \textbf{Christos Thrampoulidis}$^{1}$, \textbf{Boying Gong}$^{3\dag}$, \textbf{Xiaoxiao Li}$^{1,2\dag}$
 \vspace{8pt}
 \\
 $^1$University of British Columbia,
 $^2$Vector Institute,
 $^3$Meta
\vspace{3pt} \\
 $^\star$Work done at Meta, $^\dag$Corresponding author
}
\begin{document}
\maketitle
\begin{abstract}
Data valuation is essential for enhancing the transparency and accountability of large language models (LLMs) and vision-language models (VLMs). However, existing methods typically rely on gradient computations, making them computationally prohibitive for billion-parameter models and precluding batch parallelization. In this work, we introduce \ours{}, a forward-only data valuation framework that enables efficient batch-scalable value estimation while maintaining effectiveness. Leveraging the expressive power of pretrained LLMs/VLMs, we theoretically demonstrate that data valuation can be captured by the alignment between the final hidden representations and prediction errors at the last layer. In light of this insight, \ours{} computes data value using a simple closed-form expression with a single forward pass, eliminating the need for costly backpropagation and enabling efficient batch calculating at scale. Extensive experiments show that \ours{} matches or outperforms gradient-based baselines in detecting influential data and mislabeled data, while achieving significant efficiency improvements. Our code is available at \href{https://github.com/vengdeng/For-Value-Efficient-Forward-Only-Data-Valuation-for-finetuning}{\textcolor{pink}{GitHub}}.
\end{abstract}

\section{Introduction}
%\xl{Could you please make the problem definition more clear? See how you like my editing.}
Modern large language models (LLMs) and vision-language models (VLMs) have achieved remarkable success across a wide range of applications, driven by the power of large-scale pretraining~\citep{achiam2023gpt}. These pretrained models are subsequently fine-tuned for tasks such as machine translation, medical diagnosis, and multimodal reasoning~\citep{guo2025deepseek, bai2025qwen2, wu2025medreason, shao2024deepseekmath, hao2025can}. Despite their impressive performance, these models remain prone to generating factually incorrect or biased outputs~\citep{deng2023fairness, ferrara2023should}, often due to the presence of irrelevant, mislabeled, or unrepresentative training data. This highlights the need for scalable methods to quantify the impact of each individual training data and select the high-value samples that benefit the targeted tasks. 
% \xl{Wenlong, check this sentence. I want to put it as the first sentence of this paragraph. `The data valuation problem aims to estimate the influence of individual training instances on a model's final parameters, quantified by its effect over a valuation dataset (e.g., a validation dataset sampled from a target distribution) and serves as a proxy for desired model capabilities.'}\dwl{I feel introduce the concept of impact on parameter is deviated, how about this: The data valuation problem aims to assign a score to each training instance, quantifying its contribution to the model’s performance on a valuation set (e.g., validation data)~\cite{wang2024data}, with common performance metrics including loss, margin, or likelihood~\cite{bae2024training}.}
\begin{figure}[t]
\centering
\includegraphics[width=0.85\linewidth]{images/ee_datavalue.pdf}
\vspace{-4mm}
\caption{Comparison of data valuation methods in terms of effectiveness and efficiency when selecting training data from the Noise-Huatuo-Complex-CoT dataset for fine-tuning.}
\label{fig:ee_trade}
\vspace{-7mm}
\end{figure}
\begin{figure*}[t]
\centering
\includegraphics[width=0.95\linewidth]{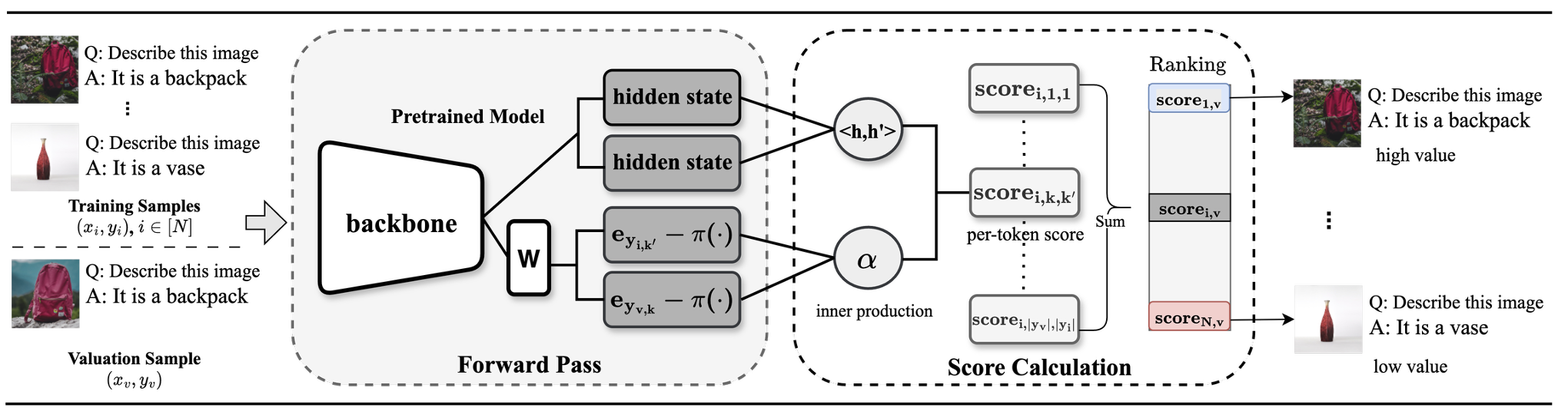} 
\caption{Pipeline of \ours{}. Given a valuation sample and a training dataset, \ours{} performs a forward pass over all data to compute scores (\cref{eq:value}) for each training example, using the last hidden embeddings and the prediction error $\alpha$. The training samples are then ranked based on these computed values.
}
\label{fig:piepline}
\vspace{-5mm}
\end{figure*}
\\
The data valuation task aims to assign scores to each training sample based on its effect on model performance on a valuation set (e.g., validation data)~\citep{wang2024data}, where performance is commonly assessed using loss, margin, or likelihood~\citep{bae2024training}. Notable approaches include influence functions~\citep{kwon2023datainf} and Shapley value-based methods~\citep{ghorbani2019data}, which provide frameworks for estimating how individual data points affect model predictions~\citep{kwon2023datainf, zhou2024hyperinf}. These methods have proven effective in downstream applications such as detecting mislabeled data~\citep{koh2017understanding, kwon2023datainf}, identifying influential examples, diagnosing bias~\citep{kong2021resolving}, and auditing datasets~\citep{grosse2023studying}. However, these methods are computationally prohibitive for V/LLMs due to reliance on Hessians or repeated retraining.
\\
To mitigate the prohibitive costs of value estimation, prior work has introduced various approximations. TracIn \citep{pruthi2020estimating} tracks first-order gradient similarity across checkpoints,  while DataInf~\citep{kwon2023datainf} and HyperInf~\citep{zhou2024hyperinf} focus on efficient Hessian approximations. However, these methods incur significant trade-offs, ranging from the massive storage requirements of TracIn to the cubic complexity and scaling errors of Hessian-based approaches, what's more, their reliance on gradient computations necessitates per-sample backpropagation, making it difficult to use large batch sizes for parallel processing.  More recently, in run Shapley value~\citep{wang2024data} calculates value along training, but still requires backpropagation and storing per-token gradients and activations, which is costly for long-sequence LLMs; additional multiplications to reconstruct gradients further limit batch scalability. Finally, while training-free similarity metrics exist for other domains \citep{just2023lava, yang2023gmvaluator}, their foundational assumptions are incompatible with the autoregressive training dynamics of LLMs and VLMs. 
\\
In this work, we challenge the necessity of gradient back-propagation for data valuation. Instead, we propose a batch-scalable, forward-only valuation framework by addressing the following core question:
\begin{center} \begin{tcolorbox} [colframe=black!50!white, colback=gray!5!white, sharp corners=all, boxrule=0.5mm, rounded corners=southeast, arc is angular ]
\textit{Can we achieve scalability via forward-only passes?}
%\\Are there scenarios where alternative \DPP methods might outperform DARE?
\end{tcolorbox}
\end{center}
%\subsection{Contributions} 
\noindent\textbf{Contributions.} We address this foundational question through a rigorous analysis of LLM/VLMs supervised finetuning learning dynamics and extensive empirical evaluation. Our contributions are as follows:

\noindent$\bullet$ \textbf{Closed-form Data Value Approximation.} We first establish that data value in LLM and VLM fine-tuning can be approximated and derived by \textit{last-layer gradient}. Under the unconstrained feature assumption, we show that the influence of a training sample on a valuation sample, measured via the last-layer gradient, \textit{admits a closed-form expression using only forward inference}: the alignment between their last-layer hidden representations, weighted by similarities in their token-level prediction errors.

%Under the unconstrained feature assumption, we formulate the influence of a training sample on a valuation sample with \textit{last-layer gradient} has a closed-form expression using forward inference only -- the alignment between their last-layer hidden representations, weighted by similarities in token-level prediction errors.

\noindent$\bullet$ \textbf{Scalability via \ours{}.} By relying only on last-layer hidden representations and token-level prediction errors obtained from a single forward pass, influence scores computed by \ours{} enable efficient large-batch parallelism and true scalability to modern LLMs and VLMs.

% \noindent$\bullet$ \textbf{Scalability via \ours{}.}  Crucially, the closed-form characterization of the last-layer gradient enables fundamental batch scalability. By relying only on last-layer hidden representations and token-level prediction errors obtained from a single forward pass, \ours{} computes influence scores without backpropagation, enabling efficient large-batch parallelism and true scalability to modern LLMs and VLMs.

\noindent$\bullet$ \textbf{Empirical Superiority in Speed and Accuracy.}  We conduct extensive experiments on both LLMs and VLMs across a range of downstream tasks (identifying influential \& mislabeled samples, and influential data finetuning). \ours{} achieves both effectiveness and superior efficiency compared to prior data valuation methods. 
% \noindent$\bullet$ \textbf{Empirical Superiority in Speed and Accuracy.}  We conduct extensive experiments on both LLMs and VLMs across a range of downstream tasks (identifying influential data, detecting mislabeled samples, and influential data finetuning). As summarized in~\cref{fig:ee_trade}, \ours{} achieves both effectiveness and superior efficiency compared to prior data valuation methods. Additional results are presented in~\cref{tab:llms_res,tab:vlm_res,tab:gsm8k,tab:medic,tab:medvl} and~\cref{fig:res_size,fig:scale_time}.
\vspace{-1mm}
\section{Related Work}
\vspace{-1mm}
\textbf{Pretrained LLMs and VLMs.} Foundation models, such as large language models and vision-language models, serve as powerful initialization points thanks to their extensive pretraining on large-scale datasets. LLMs, including LLaMA~\citep{touvron2023llama} and GPT-4~\citep{achiam2023gpt}, are trained on diverse textual data for language understanding and generation. VLMs, such as Qwen2.5-VL~\citep{bai2025qwen25vltechnicalreport}, LLaMA-VL~\citep{meta2024llama}, and GPT-4V~\citep{yang2023dawn}, integrate visual and textual inputs to perform tasks like image captioning and visual question answering. 
% Finetuning strategies have progressed from full-parameter updates to parameter-efficient finetuning methods-such as LoRA~\cite{hu2022lora} and Prefix-Tuning~\cite{li2021prefix}-which adapt models by modifying a small portion of their parameters. As finetuning becomes increasingly common~\cite{zhao2023survey}, understanding the value of each training sample becomes more crucial.
\\
\noindent\textbf{Data Valuation.}
Data valuation aims to measure the contribution of individual training examples in $\mathcal{D}^{\rm train}$ to a model’s performance on a valuation set $\mathcal{D}^{\rm val}$, typically assessed by loss or likelihood~\citep{bae2024training,wang2024data}. Most existing approaches are based on influence estimation, ranging from Hessian-based methods~\citep{koh2017understanding} to more scalable approximations such as TracIn~\citep{pruthi2020estimating}, DataInf~\citep{kwon2023datainf}, and HyperInf~\citep{zhou2024hyperinf}. Despite improved efficiency, these methods still require per-sample gradient computation during or after fine-tuning. Shapley value-based methods~\citep{ghorbani2019data} and their online variants~\citep{wang2024data} provide principled alternatives, but remain impractical due to repeated training or the need to compute and store gradients. In contrast to these gradient-dependent approaches, our method enables accurate and batch-scalable data valuation using a single forward pass.
\vspace{-1mm}
\section{Method}\label{sec:method}
\vspace{-1mm}
In this section, we provide a theoretical justification that the influence of a training sample in LLM and VLM fine-tuning is effectively captured by a closed-form score based on the last-layer gradient. We further show that this score can be effectively computed using \ours{}, a batch-scalable valuation procedure with forward only.
\\
\textbf{Notation:} Let \( \W \), \( \w_z \), and \( \hb_{\z} \) denote the token unembedding matrix, unembedding of a token \( z \in \mathcal{V} \) , where $\Vc$ is the vocabulary, and hidden embedding of generated tokens \( \z \in \mathcal{V}^* \) with embedding dimension $d$, respectively. Let \( \z_k \) be the \( k \)-th token in \( \z \) and \( \z_{<k} \) be the first \( k - 1 \) tokens in \( \z \). Lastly, we denote by \( \e_z \in \mathbb{R}^{|\mathcal{V}|} \) the standard basis vector corresponding to \( z \in \mathcal{V} \).
\\
\noindent Formally, given a training dataset ${(\x_i,\y_i)}_{i=1}^n \in \mathcal{D}^{\rm train}$ and a valuation sample $(\x_v,\y_v) \in \mathcal{D}^{\rm val}$, we define the notion of \emph{Data Value} as follows:
\begin{definition}[Data Value] At any training time $t>0$, a training sample is more valuable to a given data point \( (\x_v, \y_v) \) if it results in a greater likelihood change on valuation data $\frac{d}{dt} \ln \pi_{\theta(t)} (\y_v | \x_v)$.
\end{definition}
\noindent This definition captures how much a training sample improves the model’s confidence in predicting valuation sample $(\x_v, \y_v)$. A higher likelihood also corresponds to a lower loss on the valuation data during LLM/VLM fine-tuning. More broadly, our definition of data value is closely tied to the perplexity metric, which inversely reflects the model’s uncertainty in text generation. In this work, we focus on the pretrained initialization (\(t=0\)) and omit the time index \(t\) for brevity.
\vspace{-1mm}
\subsection{Forward-Only Data Value.}
\vspace{-1mm}
%Based on the defined data value, we then establish that the value of finetuning data can be accurately approximated by a closed-form score derived from the last-layer gradient. 
Here, we derive the closed-form score from the last-layer gradient. We first state the unconstrained feature assumption:
\vspace{-1mm}
\begin{assumption}[Unconstrained Features]\label{the:uf} Expressive (enough) neural networks (e.g., pretrained LLMs/VLMs) can produce unconstrained embeddings $\mathbf{h}_{\x} \in \mathbb{R}^d$ independent of the architecture’s specific complexities~\citep{mixon2022neural,deng2025effect,zhao2024implicit}. These embeddings are subsequently transformed into logits by a \textit{token unembedding matrix} 
$\mathbf{W} \in \mathbb{R}^{|\mathcal{V}| \times d}$. The resulting logits are passed through a softmax function to yield 
a probability distribution over possible next tokens. To assign probabilities to sequences 
$\y \in \mathcal{V}^*$, the language model $\pi_{\theta}$ operates in an autoregressive manner, \textit{i.e.},
\vspace{-1mm}
\begin{align}
\small
\pi_{\theta}(\y \mid \x) = \prod_{k=1}^{|\y|} \operatorname{Softmax}(\mathbf{W} \mathbf{h}_{\x, \y_{<k}})_{y_k}\,. \nn
\end{align}
\end{assumption}
\vspace{-1mm}
\noindent  Notably, the unconstrained feature assumption has been widely adopted in the analysis of pretrained LLMs~\citep{mixon2022neural, razin2024unintentional,zhao2024implicit}. For example, it has been leveraged in reinforcement learning studies~\citep{deng2025effect, razin2024unintentional} and in geometric analyses of LLM representations~\citep{zhao2024implicit}, reinforcing its role as a foundation for \ours{}. 
Under the unconstrained feature setting, the influence of a training sample on valuation sample is represented as (detailed proof in Appendix):
\begin{theorem}\label{the:val}
For a sample $\x_v$ and its generation $\y_v$ that await valuation, when fine-tuning a pretrained model using a training sample $(\x_i,\y_i), i \in [n]$, when no training input $\x_i$ is identical to the valuation input $\x_v$\footnote{This assumption is mild, as training inputs often differ from valuation inputs. E.g., in VLMs, images are often unique or paired with different questions. More discussion see Appx.}, the training data exhibits larger value to the valuation data as the following increases:
\begin{align}
\sum_{k=1}^{|\y_v|}  \sum_{k'=1}^{|\y_{i}|} 
\alpha_{k,k'} \cdot 
\left\langle \hb_{\x_v, \y_{v,<k}}, \hb_{\x_i, \y_{i,<k'}} \right\rangle \label{eq:value}
\end{align}
where $\alpha_{k,k'} = \big\langle 
    \mathbf{e}_{\y_{v,k}}-\pi_{\theta}(\cdot \mid \x_v, \y_{v,<k}),
    \mathbf{e}_{\y_{i,k'}} - \pi_{\theta}(\cdot \mid \x_i, \y_{i,<k'}) \big\rangle  $
quantifies the similarity of token-level prediction error across samples.
% \vspace{-2mm}
\end{theorem}
% \footnote{\Cref{the:val} extends \cite[Thm.4.4]{deng2025effect} from the GRPO setting to SFT, shifting the focus from training influence within the same question $x$ in GRPO to influence across different data points in SFT.}
As stated in the theorem, the data value arises from the alignment between hidden representations and prediction errors (effect of prediction error see \cref{sec:abstudy}). A larger score of \cref{eq:value} indicates a greater increase in the likelihood of the valuation data. Since \cref{eq:value} depends on variables resulting from forward, we refer to it as a forward-only data value, termed as \ours{}.

\vspace{-1mm}
\subsection{Scalable Forward-Only Value Calculation}
\vspace{-1mm}
We now present how to make \ours{} defined in \cref{eq:value} scalable. \Cref{fig:piepline} depicts the overall pipeline, and we detail on each component below.
\\
\noindent \textbf{Sparse Matrix Similarity:}
We first rewrite \cref{eq:value} as a matrix inner product.  Specifically, by rearranging the computation to perform the summations over $k$ and $k'$ prior to taking the inner product, the overall cost is reduced to that of a single matrix similarity operation. However, directly forming the outer product between the prediction error vector (e.g., $\mathbf{e}_{y_{i,k'}} - \pi_{\theta}(\cdot \mid \x, \y_{i,<k'})$) and the hidden embedding incurs a prohibitive computational cost of $O(|\Vc|\, d)$. 
To address this, we leverage the sparsity of the predictive distribution: we observe that probability mass is largely concentrated on words appearing in the samples. We therefore restrict the computation to a sample-associated vocabulary $\hat{\Vc}$. 
Since $|\hat{\Vc}| \ll |\Vc|$, the computational complexity is reduced to $O(|\hat{\Vc}|\, d)$ (see \Cref{tab:comparison} for a detailed efficiency analysis). Under this reformulation, the value function admits the following sparse matrix inner product:
\begin{align} \label{eq:eff_value}
\small
    &\Bigg\langle 
    \sum_{k=1}^{|\y_v|} 
    \left(\mathbf{e}_{\y_{v,k}} - \pi_{\theta} (\cdot \mid \x, \y_{v,<k}) \right)_{\hat{\Vc}}
    \hb^T_{\x_v, \y_{v,<k}},  \nn \\ 
    &
    \sum_{k'=1}^{|\y_i|} 
    \left( \mathbf{e}_{\y_{i,k'}} - \pi_{\theta} (\cdot \mid \x, \y_{i,<k'}) \right)_{\hat{\Vc}}
    \hb^T_{\x_i, \y_{i,<k'}} 
    \Bigg\rangle.
    \vspace{-4mm}
\end{align}
Moreover, when performing batch valuation, the effective vocabulary can be further restricted to the in-batch vocabulary, as illustrated in Step~6 of \Cref{alg:eff_cv}.
\\
\noindent\textbf{\ours{} Algorithm:}~\Cref{alg:for-value-batch} summarizes the efficient batch computation of \ours{}. We first obtain hidden embeddings and prediction error vectors for both the valuation and training batches using a single forward pass. We then calculate the score with sparse matrix similarity. Finally, we rank training samples by sorting their resulting influence scores. Importantly, the algorithm can be naturally extended to a group of valuation pairs by averaging their influence scores.
\begin{figure}
\vspace{-2mm}
\begin{minipage}{\linewidth}
\begin{algorithm}[H]
\caption{For-Value: Forward-Only Data Valuation}
\label{alg:for-value-batch}
\textbf{Input:} Training set $\{(\x_i, \y_i)\}_{i=1}^N$; valuation pair $(\x_v, \y_v)$; model $\pi_{\theta}$; batch size $B$. \\
\textbf{Output:} Data valuation $\mathcal{S}$.
\begin{algorithmic}[1]
\State Compute $\{\hb_{\x_v, \y_{v,<k}}\}_{k=1}^{|\y_v|}$ and $\{\pi_{\theta}(\cdot | \x_v, \y_{v,<k})\}_{k=1}^{|\y_v|}$ by doing inference $\pi_{\theta}(\x_v, \y_v)$.
\For{each batch $\{(\x_j, \y_j)\}_{j=1}^B$}
  \State Compute $\{\hb_{\x_j, \y_{j,<k'}}\}_{k'=1}^{|\y_j|}$ and \\ 
  \hspace{1.6em} $\{\pi_{\theta}(\cdot | \x_j, \y_{j,<k'})\}_{k'=1}^{|\y_j|}$ by running batch inference.
  \State $\hat{\Vc} \leftarrow \bigcup_{j=1}^{B} \Vc_{\x_j, \y_j} \cup \Vc_{\x_v, \y_v}$
  \State Compute errors $(\mathbf{e} - \pi(\cdot))$ for tokens in $\hat{\Vc}$.
  \State For each in batch, compute $S_{v,j}$ via Eq.~\eqref{eq:eff_value}.
\EndFor
\State $\mathcal{S} \leftarrow \{(\x_i, \y_i, S_{v,i})\}_{i=1}^N$.
\State Sort $\mathcal{S}$ by $S_{v,i}$ (descending).
\State \Return $\mathcal{S}$.
\end{algorithmic}
\label{alg:eff_cv}
\end{algorithm}
\end{minipage}
\vspace{-4mm}
\end{figure}
% \dwl{haven't done the quantized recall, but it should be good here for large scale data valuation.}
% \textbf{Quantized Recall:} Notably, as our \cref{eq:eff_value} calculate matrix similarity, we can further boost the calculation efficiency by flatting the matrix and using \textit{product quantization} (PQ) ~\cite{jegou2010product} to divides vectors into subvectors and independently quantizes each subvector through $Q$-means clustering. This process generates compact PQ codes that serve as representations of the original vectors. By incorporating the \emph{recall} method into the similarity matching process, the computational complexity can be further reduced.

\begin{table*}[!h]
\centering
\resizebox{0.9\textwidth}{!}{%
\begin{tabular}{l*{4}{cc}}
\toprule
\textbf{Method} 
& \multicolumn{2}{c}{\textbf{Qwen2.5-1.5B}} 
& \multicolumn{2}{c}{\textbf{Llama-2-13B-chat}} \\
\cmidrule(lr){2-3} \cmidrule(lr){4-5}
& AUC $\uparrow$ & Recall $\uparrow$ & AUC $\uparrow$ & Recall $\uparrow$ \\
\midrule

\rowcolor{gray!10} \multicolumn{5}{l}{\textbf{Sentence transformations}} \\

\hfree\citep{pruthi2020estimating} & $0.785 \pm 0.096$ & $0.370 \pm 0.139$ & $0.999 \pm 0.002$ & $0.985 \pm 0.033$ \\
\dfv\citep{kwon2023datainf} & $0.981 \pm 0.019$ & $0.826 \pm 0.121$ & $\mathbf{1.000 \pm 0.000}$ & $\underline{0.997 \pm 0.010}$ \\
\hyp\citep{zhou2024hyperinf} & $\underline{0.993 \pm 0.013}$ & $\underline{0.934 \pm 0.063}$ &$\mathbf{1.000 \pm 0.000}$ & $\underline{0.998 \pm 0.011}$ \\
\emb\citep{yang2023gmvaluator} & $0.546 \pm 0.306$ & $0.148 \pm 0.205$ & $0.854 \pm 0.192$ & $0.563 \pm 0.412$ \\
\cellcolor{blue!6}\ours{} (ours)& \cellcolor{blue!6}$\mathbf{1.000 \pm 0.001}$ & \cellcolor{blue!6}$\mathbf{0.989 \pm 0.025}$ & \cellcolor{blue!6}$\mathbf{1.000 \pm 0.000}$ & \cellcolor{blue!6}$\mathbf{1.000 \pm 0.001}$ \\
\midrule
\rowcolor{gray!10} \multicolumn{5}{l}{\textbf{Math Problem (w/o reasoning)}} \\
\hfree\citep{pruthi2020estimating} & $0.835 \pm 0.235$ & $0.592 \pm 0.291$ & $0.770 \pm 0.174$ & $0.258 \pm 0.388$ \\
\dfv\citep{kwon2023datainf} & $0.985 \pm 0.032$ & $0.878 \pm 0.154$ & $\mathbf{1.000 \pm 0.000}$ & $\underline{0.999 \pm 0.006}$ \\
\hyp\citep{zhou2024hyperinf} & $\underline{0.986 \pm 0.024}$ & $\underline{0.942 \pm 0.080}$ & $\underline{0.995 \pm 0.018}$ & $0.967 \pm 0.057$ \\
\emb\citep{yang2023gmvaluator} & $0.555 \pm 0.298$ & $0.146 \pm 0.295$ & $0.762 \pm 0.239$ & $0.389 \pm 0.477$ \\
\cellcolor{blue!6}\ours{} (ours) & \cellcolor{blue!6}$\mathbf{1.000 \pm 0.000}$ & \cellcolor{blue!6}$\mathbf{0.998 \pm 0.011}$ & \cellcolor{blue!6}$\mathbf{1.000 \pm 0.000}$ & \cellcolor{blue!6}$\mathbf{1.000 \pm 0.002}$~\footnotemark\\

\midrule
\rowcolor{gray!10} \multicolumn{5}{l}{\textbf{Math Problem (w/ reasoning)}} \\

\hfree\citep{pruthi2020estimating} & $0.829 \pm 0.172$ & $0.524 \pm 0.350$ & $0.772 \pm 0.173$ & $0.258 \pm 0.388$ \\
\dfv\citep{kwon2023datainf} & $0.987 \pm 0.030$ & $0.892 \pm 0.155$ & $\underline{1.000 \pm 0.001}$ & $\underline{0.996 \pm 0.025}$ \\
\hyp\citep{zhou2024hyperinf} & $\underline{0.988 \pm 0.023}$ & $\underline{0.950 \pm 0.060}$ & $0.994 \pm 0.018$ & $0.961 \pm 0.074$ \\
\emb \citep{yang2023gmvaluator} & $0.560 \pm 0.310$ & $0.198 \pm 0.311$ & $0.725 \pm 0.217$ & $0.270 \pm 0.420$ \\
\cellcolor{blue!6}\ours{} (ours) & \cellcolor{blue!6}$\mathbf{1.000 \pm 0.000}$ & \cellcolor{blue!6}$\mathbf{0.998 \pm 0.008}$ & \cellcolor{blue!6}$\mathbf{1.000 \pm 0.000}$ & \cellcolor{blue!6}$\mathbf{1.000 \pm 0.000}$ \\

\bottomrule
\end{tabular}%
}
\caption{Influential data identification results on LLMs. \ours{} consistently achieves comparable or superior performance. Results are reported as Mean $\pm$ Standard Deviation (std).}
\label{tab:llms_res}
\vspace{-4mm}
\end{table*}
% \vspace{-1mm}
\section{Experiment Setup}\label{sec:exp}
\vspace{-1mm}
In this section, we describe the experimental setup. More details please see Appendix.
\\
\textbf{Baseline Methods.} We focus on the comparison with baseline methods designed for efficiency. Specifically, for parameter-efficiency, we follow \dfv{}~\citep{kwon2023datainf} and \hyp{}~\citep{zhou2024hyperinf} by applying LoRA modules at each transformer layer. For the hessian efficiency, we use \texttt{Hessian-free}~\citep{pruthi2020estimating,charpiat2019input} estimates influence scores via the dot product of first-order gradients, which is equivalent to the Trace-Inf~\citep{pruthi2020estimating} or the first-order in-run Shapley~\citep{wang2024data} at the last training iteration.  \texttt{DataInf}~\citep{kwon2023datainf} simplifies the Hessian calculation by swapping the
order of the matrix inversion and the average calculations, and \hyp{}~\citep{zhou2024hyperinf} employs a low-rank Fisher approximation of the Hessian. Finally, we include an embedding similarity method~\citep{yang2023gmvaluator}, originally proposed for image generation models, denoted as \emb{}.
\\
\textbf{Models.} Following~\citet{kwon2023datainf}, we evaluate LLMs using Llama-2-13B-chat~\citep{touvron2023llama} and Qwen-2.5-1.5B~\citep{qwen2025qwen25technicalreport} to cover a wider range of model sizes and families. Moreover, thanks to the efficiency of our method, we are able to run \ours{} on Qwen2.5 series models from 7B up to 72B parameters. In contrast, baseline methods require extensive training and prolonged runtimes, making them costly for these larger models. For VLMs, we adopt the widely used Qwen-2.5-VL-3B-Instruct~\citep{bai2025qwen25vltechnicalreport} and Llama-3.2-11B-Vision~\citep{meta2024llama}.
\\
\noindent \textbf{Influential Data Identification.} We evaluate all methods on influential data identification for LLMs and VLMs, following~\citet{kwon2023datainf}. For LLMs, we use sentence transformation and math word problem datasets (w and w/o reasoning). For VLMs, we adapt image-to-text tasks from~\citet{kwon2023datainf} to an image-to-text generation setting, including style generation (cartoons, pixel art, line sketches) and subject generation using the DreamBooth dataset~\citep{ruiz2023dreambooth}. We adopt two evaluation metrics from~\citet{kwon2023datainf}: (i) AUC, measuring the correlation between data values and pseudo-labels (1 if training and valuation samples share a class, 0 otherwise), averaged over valuation points; and (ii) Recall, the proportion of top-ranked training samples sharing the same class as the valuation point. More details and dataset examples see Appendix~\Cref{sec:d_details}.
\\
\noindent \textbf{Mislabeled Data Detection.} 
We evaluate mislabeled data detection on VLMs using the Kaggle cat–dog dataset~\citep{kaggle_dogs_vs_cats}, reformulated as a QA task with 50\% label being flipped, and report AUC and Recall; examples and further details are provided in the Appendix~\Cref{sec:d_details}.
\\
\noindent \textbf{Data Selection For Finetuning.} We evaluate the practical utility of \ours{} across two key reasoning domains: mathematics and medicine. For mathematics, we use the GSM8K~\citep{cobbe2021training} dataset to assess influential data identification, while for medicine, we employ the Noise-Huatuo-Complex-CoT~\citep{chen2024huatuogpt} dataset to examine robustness under noisy training. We further extend our study to vision–language models by applying \ours{} to PMC-Reasoning~\cite{huang2025medvlthinker}. More details for each task are provided in Appendix~\cref{sec:add_fine_detail}.
\\
\noindent \textbf{Efficiency Evaluation.} For influential and mislabeled data detection with models under 32B, we compute data values using a single A100 (80G) GPU with identical hardware settings. For fine-tuning data selection, we use a single H100 (96G) GPU to calculate the data value for fair comparison. More details please see Appendix~\cref{sec:detail_time}.

\begin{table*}[th]
\centering
\resizebox{0.9\textwidth}{!}{%
\begin{tabular}{l*{4}{cc}}
\toprule
\textbf{Method} 
& \multicolumn{2}{c}{\textbf{Qwen2.5-VL-3B-Instruct}} 
& \multicolumn{2}{c}{\textbf{Llama-3.2-11B-vision}} \\
\cmidrule(lr){2-3} \cmidrule(lr){4-5}
& AUC $\uparrow$ & Recall $\uparrow$ & AUC $\uparrow$ & Recall $\uparrow$ \\
\midrule

\rowcolor{gray!10} \multicolumn{5}{l}{\textbf{Image-to-text subject generation}} \\

\hfree\citep{pruthi2020estimating} & $0.979 \pm 0.038$ & $0.738 \pm 0.399$ & $0.961 \pm 0.093$ & $0.765 \pm 0.365$ \\
\dfv\citep{kwon2023datainf} & $\underline{0.989 \pm 0.024}$ & $0.836 \pm 0.318$ & $0.958 \pm 0.119$ & $0.797 \pm 0.323$ \\
\hyp\citep{zhou2024hyperinf} & $0.988 \pm 0.047$ & \cellcolor{blue!6}$\mathbf{0.902 \pm 0.220}$ & $\underline{0.993 \pm 0.025}$ & $\underline{0.919 \pm 0.186}$ \\
\emb\citep{yang2023gmvaluator} & $0.841 \pm 0.189$ & $0.206 \pm 0.458$ & $0.841 \pm 0.189$ & $0.206 \pm 0.379$ \\
\cellcolor{blue!6}\ours{} (ours)& \cellcolor{blue!6}$\mathbf{0.994 \pm 0.018}$ & $\underline{0.897 \pm 0.287}$ & \cellcolor{blue!6}$\mathbf{0.995 \pm 0.040}$ & \cellcolor{blue!6}$\mathbf{0.985 \pm 0.068}$ \\

\midrule
\rowcolor{gray!10} \multicolumn{5}{l}{\textbf{Image-to-text style generation}} \\

\hfree\citep{pruthi2020estimating} & $0.515 \pm 0.096$ & $0.799 \pm 0.162$ & $\underline{0.515 \pm 0.079}$ & $\underline{0.824 \pm 0.145}$ \\
\dfv\citep{kwon2023datainf} & $\underline{0.520 \pm 0.094}$ & $0.760 \pm 0.181$ & $0.515 \pm 0.174$ & $0.785 \pm 0.164$ \\
\hyp\citep{zhou2024hyperinf} & $0.516 \pm 0.055$ & $\underline{0.860 \pm 0.103}$ & $0.490 \pm 0.090$ & $0.821 \pm 0.137$ \\
\emb\citep{yang2023gmvaluator} & $0.560 \pm 0.310$ & $0.198 \pm 0.311$ & $0.553 \pm 0.294$ & $0.340 \pm 0.467$ \\
\cellcolor{blue!6}\ours{} (ours)& \cellcolor{blue!6}$\mathbf{0.895 \pm 0.138}$ & \cellcolor{blue!6}$\mathbf{0.916 \pm 0.153}$ & \cellcolor{blue!6}$\mathbf{0.974 \pm 0.059}$ & \cellcolor{blue!6}$\mathbf{0.997 \pm 0.013}$ \\

\midrule
\rowcolor{gray!10} \multicolumn{5}{l}{\textbf{Mislabeled Data Detection}} \\

\hfree\citep{pruthi2020estimating} & $0.719 \pm 0.098$ & $0.760 \pm 0.088$ & $0.962 \pm 0.019$ & $0.955\pm 0.068$ \\
\dfv\citep{kwon2023datainf} & $0.760 \pm 0.088$ & $0.901 \pm 0.147$ & \cellcolor{blue!6}$\mathbf{1.000 \pm 0.000}$ & $\underline{1.000 \pm 0.003}$ \\
\hyp\citep{zhou2024hyperinf} & $\underline{0.770 \pm 0.077}$ & $\underline{0.916 \pm 0.128}$ & $\underline{1.000 \pm 0.001}$ & $1.000 \pm 0.006$ \\
\emb\citep{yang2023gmvaluator} & $0.741 \pm 0.061$ & $0.533 \pm 0.075$ & $0.933 \pm 0.044$ & $0.996 \pm 0.015$ \\
\cellcolor{blue!6}\ours{} (ours)& \cellcolor{blue!6}$\mathbf{0.885 \pm 0.055}$ & \cellcolor{blue!6}$\mathbf{0.999 \pm 0.010}$ & $0.995 \pm 0.008$ & \cellcolor{blue!6}$\mathbf{1.000 \pm 0.000}$ \\

\bottomrule
\end{tabular}%
}
\caption{Influential data identification and mislabeled data detection performance for different VLM tasks. \ours{} consistently delivers comparable or superior performance in identifying influential data and detecting mislabeled data across various VLM tasks compared to baseline methods.}
\label{tab:vlm_res}
\vspace{-4mm}
\end{table*}
% \vspace{-1mm}
\section{Results}
\vspace{-1mm}
In this section, we detail the results of \ours{} and baselines on LLMs and VLMs.
\vspace{-1mm}
\subsection{Identify Influential \& Mislabeled Data}\label{sec:im_res}
\noindent \textbf{Influential data identification Results on LLM.}
We first present the results for text generation tasks in~\Cref{tab:llms_res}, where \ours{} consistently matches or outperforms all baseline methods across the evaluated LLM benchmarks:
\\
(1) \textit{Sentence Transformation:} As shown in ~\Cref{tab:llms_res}, \ours{} achieves perfect or near-perfect AUC and recall scores for both models. Notably, on Qwen2.5-1.5B, \ours{} surpasses the strongest baseline \hyp{} by 6.5\% in recall.
\\
(2) \textit{Math Problems (w/\&w/o reasoning):} A similar pattern holds for the math task. As shown in~\Cref{tab:llms_res}, \ours{} delivers higher-quality influence identification with just a single forward pass, improving recall by 6\% over \hyp{} on both math datasets with the Qwen model. 
\footnotetext{AUC and Recall values reported as 1.0 may still include a non-zero std due to rounding. The large std arises since value distribution is highly polarized, clustering near either 1 or 0.}
\\
\noindent\textbf{Influential data identification Results on VLM.} We next report the results on VLMs in~\Cref{tab:vlm_res}.
(1) For subject generation, \ours{} achieves the highest AUC and recall scores for both models, consistently outperforming all baselines. Specifically, \ours{} exceeds the strongest baseline, \hyp{}, by more than 7\% in recall for both models for the 11B model.
(2) In the more challenging style generation task, \ours{} demonstrates a clear advantage, with AUC improvements of over 0.35 compared to the baselines, and even larger gains over the \emb{} method. 
\\
\textbf{Mislabeled Data Detection.} Our mislabeled data detection results in \Cref{tab:vlm_res} demonstrate \ours{}'s strong performance across model scales. On the Qwen-VL-3B model, \ours{} achieves an 11.5\% higher AUC and an 8.3\% higher Recall compared to the best baseline (\hyp{}), showing significant improvements in identifying mislabeled examples. The method performs equally well on the larger Llama-3.2-11B model, matching the near-perfect detection rates (AUC $>$ 0.99, Recall = 1.0) of gradient-based approaches. This consistent performance across both models highlights For-Value's effectiveness. Additional results of \ours{} under varying noise ratios are reported in~\cref{tab:noise_ab}.
\vspace{-1mm}
\subsection{Data Selection For Finetuning}
\vspace{-1mm}
We next assess \ours{}'s practical utility on mathematics and medicine. Given poor performance of \emb{} in prior experiments, we excluded it from these evaluations.
% Having established the strong performance of \ours{} in identifying both influential and noisy data, we next assess its practical utility on mathematics and medicine. 
\begin{table}
\centering
\resizebox{\linewidth}{!}{%
\begin{tabular}{l c c c}
\toprule
\textbf{Llama-3.1-8B} & \textbf{GSM8K (1\%) $\uparrow$} & \textbf{GSM8K (5\%) $\uparrow$} & \textbf{Time $\downarrow$} \\
\midrule
Full (100\%) & \multicolumn{2}{c}{47.8} & -- \\
\midrule
\hfree{} & 41.5 & 41.8 & 1.4 h \\
\hyp{}   & 41.9 & 42.8 & 2.4 h \\
\dfv{}   & 41.7 & 42.0 & 1.9 h \\
\ours{} (ours)  & \textbf{45.2} & \textbf{48.3} & \textbf{0.3 h} \\
\bottomrule
\end{tabular}
}
\caption{GSM8K greedy decoding accuracy of Llama-3.1-8B. Best results are in \textbf{bold}.}
\label{tab:gsm8k}
\vspace{-5mm}
\end{table}
\begin{table*}[ht]
\centering
\resizebox{0.9\textwidth}{!}{%
\begin{tabular}{l c c c c c c c}
\toprule
\textbf{Method (Llama-3.1-8B-Ins)} & 
\textbf{MedQA} & 
\textbf{MedMCQA} & 
\textbf{PubMedQA} & 
\textbf{MMLU-Pro-med} & 
\textbf{GPQA-med} & 
\textbf{Average $\uparrow$} & 
\textbf{Time $\downarrow$} \\
\midrule
Base & 56.84 & 61.90 & 77.00 & 59.02 & 44.35 & 59.82 & -- \\
\midrule
\multicolumn{8}{c}{\textit{5\% Data}} \\
\midrule
\hfree{} & 55.41 & 58.05 & 73.40 & 54.53 & 38.46 & 55.97 & 2.0 (h) \\
\hyp{}     & 55.15 & 57.58 & 71.50 & 54.14 & 43.08 & 56.29 & 5.3 (h) \\
\dfv{}     & 55.39 & 57.74 & 73.30 & 54.07 & 45.13 & 57.13 & 4.1 (h) \\
\ours{}  (ours) & \textbf{56.80} & \textbf{62.92} & \textbf{77.60} & \textbf{58.31} & \textbf{45.90} & \textbf{60.31} & \textbf{0.8 (h)} \\
\midrule
\multicolumn{8}{c}{\textit{10\% Data}} \\
\midrule
\hfree{} & 57.02 & 59.15 & 72.30 & 57.13 & 47.69 & 58.66 & 2.0 (h) \\
\hyp{}      & 56.94 & 62.76 & 77.40 & 57.85 & 48.46 & 60.28 & 5.3 (h) \\
\dfv{}       & 56.61 & 61.74 & 75.60 & 56.81 & 43.85 & 58.92 & 4.1 (h) \\
\ours{} (ours) & \textbf{57.61} & \textbf{67.16} & \textbf{78.30} & \textbf{58.18} & \textbf{50.51} & \textbf{62.35} & \textbf{0.8 (h)} \\
\bottomrule
\end{tabular}}
\caption{Results of data selection for fine-tuning on the Noise Huatuo-Complex-CoT (Llama-3.1-8B-Ins).}
\label{tab:medic}
\vspace{-3mm}
\end{table*}
\begin{table*}[h]
\centering
\resizebox{0.9\textwidth}{!}{%
\begin{tabular}{lcccccccc}
\midrule
\textbf{Method (Qwen2.5-VL-3B)} & \textbf{MMMU} & \textbf{MedX-M} & \textbf{PathVQA} & \textbf{PMC} & \textbf{SLAKE} & \textbf{VQA-Rad} & \textbf{Average $\uparrow$} & \textbf{Time $\downarrow$} \\
\hline
Base & 44.12 & 20.69 & 61.96 & 44.77 & 61.30 & 62.01 & 49.14 & -- \\
Full$^{\star}$~\citep{huang2025medvlthinker} & 47.84 & 21.46 & 52.76 & 54.55 & 65.79 & 58.58 & 50.16 & -- \\
\hline
\multicolumn{9}{c}{\textit{10\% Data}} \\
\hline
\hfree{}   & 48.82 & 20.65 & 61.18 & 49.60 & 61.78 & 63.60 & 50.94 & 1.3 (h) \\
\hyp{}     & \textbf{50.00} & 21.60 & 61.10 & 50.45 & 62.50 & 63.97 & \underline{51.60} & 1.7 (h) \\
\dfv{}     & 49.41 & 21.10 & 62.64 & \textbf{50.55} & 59.38 & \textbf{65.81} & 51.48 & 1.6 (h) \\
\ours{} (ours)   & 47.06 & \textbf{23.05} & \textbf{62.93} & 49.55 & \textbf{67.55} & 63.24 & \textbf{52.23} & \textbf{0.4 (h)} \\
\hline
\multicolumn{9}{c}{\textit{20\% Data}} \\
\hline
\hfree{}   & 52.94 & 21.40 & 61.81 & 52.05 & 63.46 & 62.50 & 52.36 & 1.3 (h) \\
\hyp{}     & \textbf{56.47} & 20.50 & \textbf{62.14} & \textbf{51.45} & 62.98 &  \textbf{64.71} & \textbf{53.04} & 1.7 (h) \\
\dfv{}     & 48.82 & 21.25 & 62.58 & 51.35 & 63.46 & 63.24 & 51.78 & 1.6 (h) \\
\ours{} (ours)    & 54.12 & \textbf{22.45} & 60.26 & 50.45 & \textbf{65.14} &63.60 & \underline{52.67} & \textbf{0.4 (h)} \\
\hline
\end{tabular}
}
\caption{Results of data selection for fine-tuning on the PMC-Reasoning dataset. Best results are in \textbf{bold}, and second-best are \underline{underlined}. $^{\star}$ denotes results from~\citep{huang2025medvlthinker}.}
\label{tab:medvl}
\vspace{-4mm}
\end{table*}
\\
\textbf{Mathematics: GSM8K.} We evaluate influential data identification on GSM8K~\citep{cobbe2021training} using test reasoning examples as valuation data, where high-value training samples are expected to improve test accuracy. Following~\citep{deng2024dare}, we report greedy fine-tuning results in~\cref{tab:gsm8k}. Fine-tuning on the top 5\% samples selected by \ours{} achieves the highest accuracy of 48.3\%, surpassing the strongest baseline, \hyp{}, by 5.5\%, and slightly outperforming training on the full dataset, as expected when using test data for valuation. Reducing the selection rate to 1\% lowers performance, but \ours{} still exceeds all baselines by up to 3.3\%. Importantly, \ours{} is over $5\times$ faster than all baselines.
\\
\textbf{Medicine: Noise-Huatuo-Complex-CoT.} To examine robustness under noisy training conditions, we construct a corrupted version of the Huatuo-Complex-CoT dataset~\citep{chen2024huatuogpt}. We randomly sample 5,000 examples without replacement and inject noise into 40\% of them by inserting or removing irrelevant words~\cite{wei2019eda} (examples see \cref{fig:noisy} in Appendix), resulting in the Noise-Huatuo-Complex-CoT dataset. Another 5,000 clean examples are reserved for valuation, and models are evaluated on five held-out medical QA test sets. Within this setting, we apply \ours{} and competing methods to select high-quality training subsets for fine-tuning. As shown in~\Cref{tab:medic}, \ours{} consistently delivers the strongest results. With only 5\% data, it reaches an average accuracy of 60.31\%, outperforming the best baseline (\dfv{}) by 3\%. At 10\%, \ours{} shows an even clearer advantage, achieving the best score across all tasks with an average of 62.35\%, exceeding the strongest baseline \hyp{} by 2.1\%. Crucially, \ours{} also provides the most efficient valuation, requiring only 0.8h, up to $6\times$ faster than baselines. These results underscore the effectiveness of \ours{} in identifying valuable data even when training data is noisy. More analysis see Appendix~\cref{sec:select_add}.
\\
\textbf{Medical VQA.} To evaluate the effectiveness of \ours{} on vision–language models, we conduct experiments on the PMC-Reasoning dataset~\citep{zhang2023pmc}. We randomly sample 10,000 examples for training and 5,000 for valuation without replacement. Fine-tuning subsets are then selected from the training pool using \ours{} as well as baseline methods, and the resulting models are fine-tuned and evaluated on six held-out test sets. As shown in~\cref{tab:medvl}, \ours{} delivers the strongest overall performance. With 10\% data, it achieves the highest average accuracy (52.23\%), exceeding the base model by over 3\% and the best-performing baseline, \hyp{}, by 0.6\%. At 20\% data, \ours{} maintains competitive performance (52.67\%), ranking second only to \hyp{}. Importantly, \ours{} consistently achieves these results with the lowest computational cost (0.4h vs. 1.6–1.7h for baseline methods). Notably, all data valuation methods surpass full fine-tuning, highlighting the benefit of selecting high-value subsets for training. Overall, \ours{} reliably identifies influential medical VQA data while offering large efficiency gains.
% \begin{table}
% \centering
% \resizebox{\textwidth}{!}{%
% \begin{tabular}{l|c|c|c|c|c|c|c}
% \hline
% Llama-3.1-8B-Ins & MedQA(Acc)  & MedMCQA(Acc)  & PubMedQA(Acc)  & MMLU-Pro-med (Acc)   & GPQA-med(Acc) $\uparrow$ & Average & Time $\downarrow$  \\
% \hline
% Base (100\%) & 0.5637 & 0.6575 & 0.742 & 0.569 & 0.4949 & 0.6054 & - \\
% Random (5\%) & 0.5596 & 
%  0.6442  & 
%  0.7440  & 
%  0.5616  & 
% 0.4487  & 
% 0.5916&  - \\
% \hfree{} (5\%)  &  0.5653 & 0.6528 & 0.76 & 0.5739 & 0.4410 & 0.5963
%  & 2.5 (h) \\
% \hyp{} (5\%) & 0.5649 & 0.6740 & 0.763 &0.5726 & 0.4923 & 0.6086 &  8.4  (h)\\
% \dfv{} (5\%) & 0.5671 & 0.6528 & 0.755 &0.5700 &0.4795 & 0.6078
% & 7.1 (h)\\
% \ours{} (5\%) &  0.5685 & 0.6481  & 0.741 & 0.5707 & 0.4821 & 0.6028
%   & \textbf{1.2 (h)} \\
% \hline
% \end{tabular}
% }
% \label{tab:medical}
% \caption{Huatuo-Complex-CoT (Llama3-8B) Accuracy Results}
% \end{table}
\vspace{-1mm}
\subsection{Ablation Study \& Efficiency}\label{sec:abstudy}
\vspace{-1mm}
We present ablation and efficiency studies based on influential and mislabeled data identification tasks.
\\
\textbf{Effect of prediction error similarity $\alpha$.} We perform an ablation study to evaluate the role of the $\alpha$ term by setting $\alpha$ to 1 in the computation of~\Cref{eq:eff_value}. This simplification reduces the score to $\left\langle \sum_{k=1}^{|\y_v|} \hb_{\x_v, \y_{v,<k}}, \sum_{k'=1}^{|\y_i|} \hb_{\x_i, \y_{i,<k'}} \right\rangle $, which measures contextualized text embedding similarity between two data samples' $\y$ and is equivalent to the \emb{} baseline. As shown in~\Cref{tab:llms_res,tab:vlm_res}, \ours{} consistently outperforms \emb{} across both LLM and VLM tasks, highlighting the importance of the $\alpha$ term. The prediction error in $\alpha$ acts as a token-level weight, down-weighting tokens that are already confidently predicted in either the training or valuation data. While \emb{} is effective for generative image models, its degraded performance indicates that directly applying embedding similarity to LLMs/VLMs is suboptimal due to differing training objectives.
\\
\noindent \textbf{\ours{} Performance across model sizes.} 
\begin{figure}
\centering
\includegraphics[width=\linewidth]{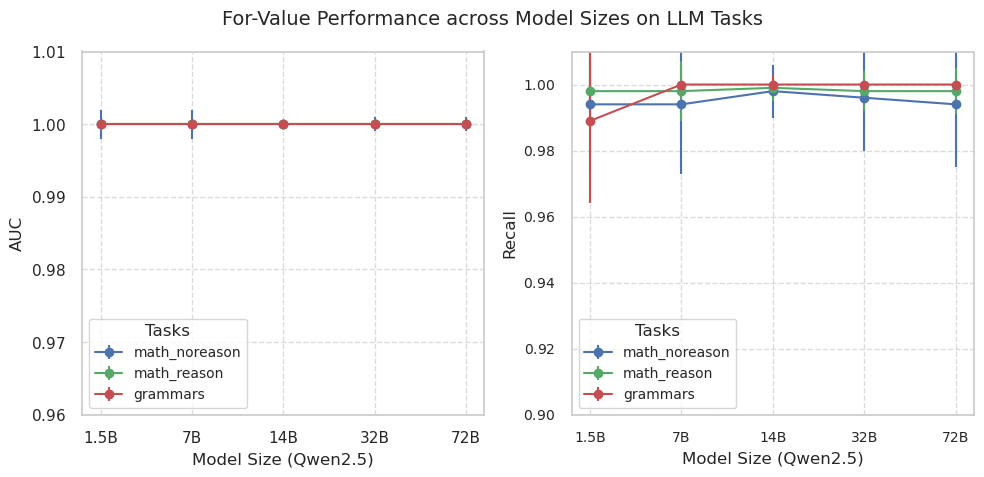} 
\caption{\ours{} performance across model sizes and tasks (Mean$\pm$std).}
\label{fig:res_size}
\vspace{-6mm}
\end{figure}
\Cref{fig:res_size} shows that \ours{} maintains consistently high performance across different model sizes and tasks. Both AUC and Recall stay close to $1.0$ for all tasks, indicating that scaling up the model does not degrade effectiveness. This stability confirms that \ours{} generalizes well to larger models while preserving accuracy, making it reliable for practical deployment on LLM tasks.
\begin{figure}
    \vspace{-2mm}
    \centering
    \begin{subfigure}[b]{0.465\linewidth}
        \centering        \includegraphics[width=\linewidth]{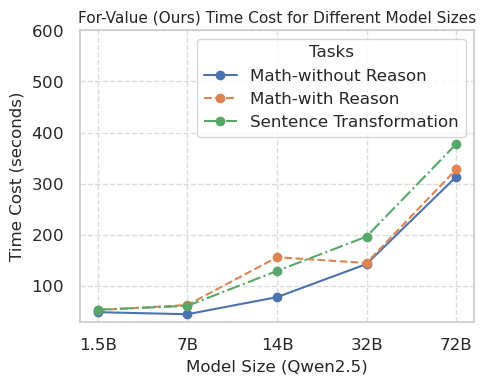}
        \caption{\ours{}.}
        \label{fig:for_time}
    \end{subfigure}
    \begin{subfigure}[b]{0.52\linewidth}
        \centering        \includegraphics[width=\linewidth]{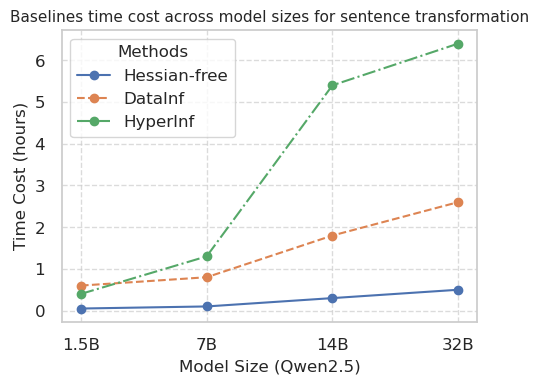}
        \caption{Baseline Methods}
        \label{fig:base_time}
    \end{subfigure}
    \caption{Time cost analysis: (a) \ours{} across different model sizes and tasks. (b) Baseline methods on the sentence transformation task across model sizes. \ours{} is more efficient with time costs measured in seconds, whereas baselines require up to several hours.}
    \label{fig:figures}
 \label{fig:scale_time}
   \vspace{-6mm}
\end{figure}
\\
\textbf{Effectiveness of last-layer gradient.} The Hessian-free baseline follows the setup of DataInf~\cite{kwon2023datainf}, computing gradient similarity over LoRA parameters across all transformer layers. In contrast, \ours{} relies solely on the last-layer gradient. As shown in~\Cref{tab:llms_res,tab:vlm_res,tab:gsm8k,tab:medic}, \ours{} consistently matches or outperforms \hfree{}, demonstrating that the last-layer gradient provides an effective measure of data value and further supporting our theoretical analysis in~\Cref{the:val}.
\\
\textbf{Time Cost Analysis.} We compare the time cost of \ours{} with that of the baselines across model sizes using influential data identification tasks on a single GPU. As shown in \Cref{fig:for_time}, \ours{} maintains consistently low runtime, even as model size increases from 1.5B to 72B parameters. For all tasks, the runtime remains within a few hundred seconds, highlighting its practical scalability. In contrast, as shown in \Cref{fig:base_time}, baseline methods for the sentence transformation task require more time—measured in hours rather than seconds. The best-performing baseline, \hyp{}, becomes costly for larger models, taking 6 hours for the 32B model. This underscores the efficiency advantage of \ours{}. More details see~\cref{sec:add_res}.

% \vspace{-1mm}
\section{Conclusion}
\vspace{-1mm}
We theoretically show that data influence can be accurately approximated by the alignment between last-layer hidden representations and token-level prediction errors, eliminating the need for backpropagation. Building on this insight, we introduce \ours{}, a forward-only data valuation framework for pretrained LLMs and VLMs. Using a single forward pass, \ours{} matches existing methods in identifying influential and mislabeled data and in selecting high-value subsets for fine-tuning, while being substantially more efficient.

\section{Limitations}
Our method is tailored to data valuation in the fine-tuning stage and is not directly applicable to pretraining data selection, where the unconstrained feature assumption may not hold. Nevertheless, fine-tuning remains a critical stage for adapting pretrained models to downstream tasks. In addition, data value may evolve over the course of training. Extending \ours{} to support stage-aware data selection or active learning is left for future work, though we believe such extensions can be incorporated naturally.

\section{Acknowledgments}
Wenlong Deng, Christos Thrampoulidis, and Xiaoxiao Li acknowledge support from the NSERC; NSERC Discovery Grant RGPIN-2022-05316, RGPIN-2021-03677; NSERC Alliance Grant ALLRP 602633-24, 581098-22; Tri-Agency Canada; Canada CIFAR AI Chair Awards; Canada Research Chair Fellowship; IITP grant; the Ministry of Science and ICT (RS-2024-00445087, RS2025-25464461).

\bibliography{main}
\clearpage
\onecolumn
\addcontentsline{toc}{section}{Appendix}
\tableofcontents
\twocolumn
\appendix
\section{More Details and Results}

\subsection{Additional Details}\label{sec:detail_time}
\textbf{Training setting for baselines.} While \ours{} requires only a single forward pass, the influence function-based baselines \hfree{} and \dfv{} require fine-tuning the models to convergence. For text generation tasks, we follow the training setup in~\cite{kwon2023datainf}, except to llama-2-13B, we use float16 weights instead of 8-bit quantization. For image-to-text generation tasks, we apply LoRA to every query and value matrix within the model’s attention layers. To finetune VLMs, we use a learning rate of $2 \times 10^{-4}$, LoRA hyperparameters $r=8$ and $\alpha=32$, float16 model weights, a batch size of 32, and train for 20 epochs. 
\\
\textbf{Efficiency details.} For \ours{}, we use a single A100 for small models. For larger 32B and 72B models in~\cref{fig:scale_time}, inference is run on 4 A100 GPUs, while value computation is performed on a single A100. Baseline methods that require training are fine-tuned using up to 8 GPUs (for 32B models); valuation is still conducted on one A100, with 8-bit quantization applied to Qwen-32B to satisfy memory constraints. Due to their high computational cost, baselines are evaluated only on the sentence transformation task, and for 14B and 32B models we subsample 10\% of the valuation data and scale the measured runtime by a factor of 10 to estimate total cost. Despite these favorable conditions, \ours{} achieves substantially lower runtime without quantization and using fewer GPUs. For data selection experiments, all methods are evaluated on a single H100 to ensure a fair comparison.

\subsection{Additional Results}\label{sec:add_res}
\textbf{Complexity Analysis.} \Cref{tab:comparison} compares the training, computational, and memory costs of different methods. Traditional approaches such as IF, \hfree{}, \hyp{}, and \dfv{} rely on gradient traces or Hessian computations, resulting in high costs that scale poorly with model size. In contrast, \emb{} and \ours{} are training-free and algorithm-agnostic, which significantly reduces overhead. Although \hyp{} is the strongest baseline in terms of accuracy, its cubic complexity makes it impractical for large LLMs—requiring about 6 hours for a Qwen-32B model (\Cref{fig:base_time}). Although \emb{} achieves the best runtime efficiency, its performance lags behind other methods, as demonstrated in \Cref{tab:llms_res} and \Cref{tab:vlm_res}. Our method, \ours{}, maintains strong performance while remaining highly efficient. Since $|\hat{\Vc}|$ is typically small (often under 2k), \ours{} achieves much lower computational and memory costs than baselines.
\begin{table*}[h!]
\centering
\resizebox{\textwidth}{!}{%
\begin{tabular}{c|c|c|c|c|c}
\hline
Method & Training Free & Algorithm Agnostic & Training Complexity & Computational Complexity & Memory Complexity  \\
\hline
Original IF & \ding{55} & - & $O(nEd_{in}dL)$ & $O(nd^2_{in}d^2L + d^3_{in}d^3L)$ & $O(D^2L + nDL)$\\
\hfree{} & \ding{55} & \ding{55} & $O(nEd_{in}dL)$ & $O(nd_{in}dL)$ & $O(nd_{in}dL)$   \\
\dfv{} & \ding{55} & \ding{55} & $O(nEd_{in}dL)$ & $O(nd_{in}dL)$ & $O(nd_{in}dL)$ \\
\hyp{} & \ding{55} & \ding{55} & $O(nEd_{in}dL)$ & $O(nd^3L)$ & $O(nd^2L)$  \\
\emb{} & \ding{51} & \ding{51} & 0 & $O(nd)$ & $O(nd)$  \\
\ours{} (ours) & \ding{51} & \ding{51} & 0 & $O(nd|\hat{\Vc}|)$  & $O(nd|\hat{\Vc}|)$  \\
\hline
\end{tabular}
}
\caption{
Comparison on complexity of the Influence Function (IF), \hfree{}, \dfv{}, \emb{}, and \ours{}. Complexities are given assuming a multilayer perceptron (MLP) with $L$ layers, each containing $d_{in}\times d$ neurons where $d_{in}$ is input dimension and $d$ is the output embedding dimension, trained for $E$ epochs on $n$ training samples. The parameter count is identical across layers ($D \in \mathbb{N}$), and the in-batch volcabulary size is $|\hat{\Vc}|$. Overall, \ours{} achieves higher computational and memory efficiency than baseline methods. 
%Although \emb{} is more efficient, its performance is limited in autoregressive settings.
}
\label{tab:comparison}
%\vspace{-5mm}
\end{table*}
\\
\textbf{Discussion on Parallel Computing:} While previous studies focus on using a single GPU for fair comparison, we would like to highlight that \ours{} can further improve efficiency through parallel computing with a large batch size, as it only requires forward calculations. In contrast, baseline methods require computing the gradient for each individual data sample, which restricts them to a batch size of one and makes scaling up challenging.
\\
\textbf{Qualitative Demonstration.}
Beyond quantitative results, we present qualitative examples identified by \ours{}. \Cref{fig:example} shows a target valuation sample alongside its most and least influential training samples as ranked by \ours{}. Specifically, \ours{} successfully identifies highly relevant training points — for example, selecting samples from the same reverse order of words task for sentence transformation, or matching the same subject or artistic style in image-to-text tasks. In contrast, the least influential samples are clearly less relevant and often differ entirely in task or content from the target valuation data. 
\begin{figure*}[h]
\centering
\includegraphics[width=1.0\linewidth]{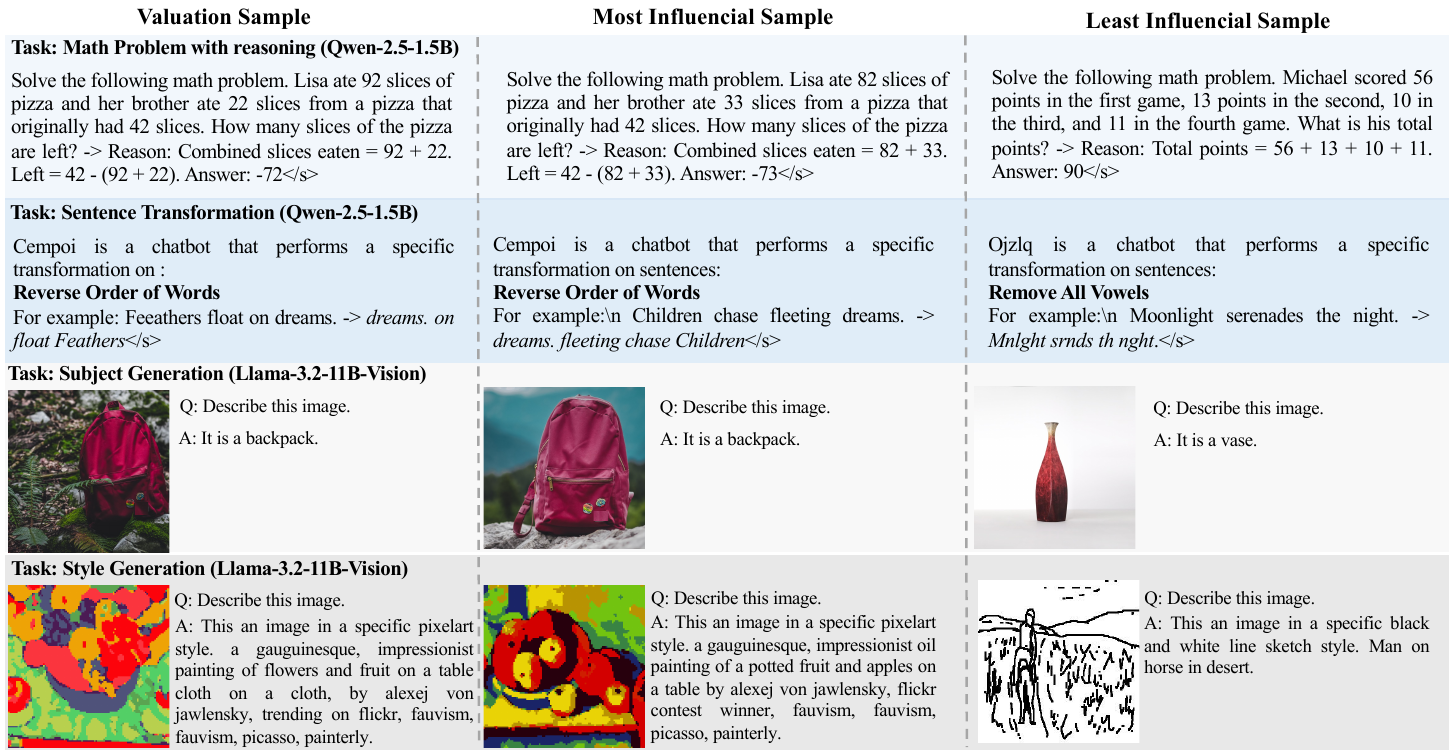} 
\caption{Qualitative examples of data influence identified by \ours{}. For each target valuation sample (left column), the most influential (middle column) and least influential (right column) training samples are shown. \ours{} correctly retrieves training samples that share relevant task characteristics (e.g., same reasoning type, sentence transformation rule, subject, or style) and filters out unrelated or mismatched examples.
}
\label{fig:example}
%\vspace{-3mm}
\end{figure*}
\\
\textbf{Performance of \ours{} on different Mislabeled Ratio}
We evaluate the performance of \ours{} under varying mislabeling ratios. As shown in~\cref{tab:noise_ab}, \ours{} consistently achieves high mislabeled data detection rates, demonstrating its robustness and effectiveness across different noise levels.

\begin{table}[t]
\centering
\label{tab:img2txt_subject}
\resizebox{0.47\textwidth}{!}{%
\begin{tabular}{lcc}
\toprule
\multirow{2}{*}{\textbf{Mislabel Ratio}} & \multicolumn{2}{c}{\textbf{Qwen2.5-VL-3B-Instruct}} \\
\cmidrule(lr){2-3}
 & \textbf{AUC $\uparrow$} & \textbf{Recall $\uparrow$} \\
\midrule
0.4 & $0.853 \pm 0.048$ & $0.972 \pm 0.123$  \\
0.5 & $0.885 \pm 0.055$ & $0.999 \pm 0.010$  \\
0.6 & $0.902 \pm 0.046$ & $0.952 \pm 0.150$  \\
\bottomrule
\end{tabular}
}
\caption{Performance of \ours{} on Mislabeled data detection with different noise ratio using Qwen2.5-VL-3B-Instruct.}
\label{tab:noise_ab}
\vspace{-3mm}
\end{table}

\subsection{Additional Details of Select Data for Finetuning}\label{sec:add_fine_detail}

\textbf{Mathematics: GSM8K} As the baseline methods require LoRA, we begin with a one-epoch warmup training on Llama3-8B~\cite{meta2024llama} using the whole training set to avoid utilizing gradients from randomly initialized LoRA modules (with a rank of $r=32$). Next, we calculate influence scores for both the baselines and \ours{}. To ensure consistency and performance, we also perform a one-epoch warm-up but with full-parameter finetuning on the entire dataset. Finally, we select the top 5\% of data based on these influence scores to further finetune the model with learning rate $1e-5$ and batch size 64 on 4 H100 GPU for 4 epochs.
\\
\textbf{Medicine: Noise-Huatuo-Complex-CoT} As the baseline methods utilize LoRA, we begin with a one-epoch training on Llama3-8B-Instruction~\cite{meta2024llama} using the whole training set to avoid using gradients from randomly initialized LoRA modules (with a rank of $r=16$). Next, we calculate influence scores for both the baselines and our approach. Considering the training data is noisy, we select the top 5\% high value training data based on these scores and finetune the original pretrained model using full-parameter finetuning for 5 epochs, with a learning rate of $1 \times 10^{-6}$, a batch size of 16 and gradient accumulation 8 on 8 H100 GPUs. We follow~\cite{wu2025medreason} using greedy decoding to evaluate the model on 5 held out datasets MedQA~\cite{jin2021disease}, MedMCQA~\cite{pal2022medmcqa}, PubMedQA~\cite{jin2019pubmedqa}, MMLU-Pro-Med~\cite{wang2024mmlu}, GPQA-Med~\cite{rein2024gpqa}.
\\
\textbf{Medicine: PMC-Reasoning} Similarly, we start with a one-epoch warm-up on the entire training set to prevent using gradients from randomly initialized LoRA modules (with a rank of $r=16$). Then, we compute influence scores for the baseline methods. For our method, since the pretrained model already demonstrates sufficient medical knowledge (as shown by adequate test accuracy in Table \ref{tab:vlm_res}), we directly use the original pretrained model to assess data value. Finally, we finetune the pretrained Qwen2.5-3B-VL model~\cite{bai2025qwen25vltechnicalreport} with full-parameter finetuning for 3 epochs, using a learning rate of $1 \times 10^{-5}$, a batch size of 16, and gradient accumulation of 8 on 8 H100 GPUs. We evaluate the model with greedy decoding on 6 held out datasets: PMC~\cite{zhang2023pmc}, MMMU~\cite{yue2024mmmu}, MedX-M~\cite{zuo2025medxpertqa}, PathVQA~\cite{he2020pathvqa}, SLAKE~\cite{liu2021slake}, VQA-Rad~\cite{lau2018dataset}.

\subsection{Additional Analysis on Select Data for Finetuning} \label{sec:select_add}

\textbf{Medicine: Noise-Huatuo-Complex-CoT}. As indicated in \cref{tab:medic}, baseline methods struggle to effectively select high-quality data from noisy training datasets. This is primarily because these methods rely on assumptions of uniqueness or convergence to an optimal solution~\cite{bae2024training}, which are difficult to satisfy in the presence of noisy data. To illustrate this, we evaluated the proportion of high-quality data within the top 10\% of high-value data, as shown in \cref{tab:clean_p}. The results reveal that baseline methods generally lack the capability to accurately identify noisy data, whereas our proposed method (\ours{}) achieves significantly higher accuracy in detecting clean data. 
\begin{table}
% \vspace{-4mm}
\centering
\resizebox{0.47\textwidth}{!}{%
\begin{tabular}{l|c}
\hline
Llama-3.1-8B & Detection Accuracy  \\
\hline
\hfree{} & 48.2 \\
\hyp{} & 15.1 \\
\dfv{} & 33.2 \\
\ours{} & 84.4 \\
\hline
\end{tabular}
}
\caption{High quality data detection accuracy}
\label{tab:clean_p}
\vspace{-3mm}
\end{table}

\begin{table*}[h!]
\centering
\caption{Description of the sentence transformation task templates. We consider 10 different types of sentence transformations. For each sentence transformation, unique identifying ``chatbot'' names were additionally prepended to the task prompt to assist the model in training.}
\label{tab:sentence_transformations}
\resizebox{\textwidth}{!}{%
\begin{tabular}{l|l}
\hline
\hline
\textbf{Sentence transformations} & \textbf{Example transformation of ``Sunrises herald hopeful tomorrows'':} \\
\hline
Reverse Order of Words & tomorrows. hopeful herald Sunrises \\
\hline
Capitalize Every Other Letter & sUnRiSeS hErAlD hOpEfUl tOmOrRoWs. \\
\hline
Insert Number 1 Between Every Word & Sunrises 1herald 1hopeful 1tomorrows. \\
\hline
Replace Vowels with * & S*nr*s*s h*r*ld h*p*f*l t*m*rr*ws. \\
\hline
Double Every Consonant & SSunrriisseess hheraldd hhopefull ttomorrows. \\
\hline
Capitalize Every Word & Sunrises Herald Hopeful Tomorrows. \\
\hline
Remove All Vowels & Snrss hrld hpfl tmrrws. \\
\hline
Add 'ly' To End of Each Word & Sunrisesly heraldly hopefully tomorrows.ly \\
\hline
Remove All Consonants & uie ea oeu ooo. \\
\hline
Repeat Each Word Twice & Sunrises Sunrises herald herald hopeful hopeful tomorrows. tomorrows. \\
\hline
\hline
\end{tabular}
}
\label{tab:trans}
\end{table*}

\subsection{Detailed Task Description}\label{sec:d_details}
\subsubsection{LLM Influence Evaluation Tasks}
Following~\citep{kwon2023datainf}, we evaluate the performance of \ours{} on three text generation tasks for large language models (LLMs) to identify influential data points:

\begin{itemize}
    \item \textbf{Sentence Transformations:} This task requires transforming input sentences into alternative forms while preserving meaning (e.g., active to passive voice). The dataset comprises 10 distinct classes (e.g., declarative to interrogative), each with 100 examples, split into 90 training and 10 test examples per class. Data examples see \Cref{tab:sentence_transformations}.
    \item \textbf{Math Word Problems (Without Reasoning):} These problems involve direct numerical computation from textual descriptions (e.g., basic arithmetic). The dataset has 10 classes based on operation types, with 100 examples per class (90 training, 10 test). Data examples see \Cref{tab:math_problems}.
    \item \textbf{Math Word Problems (With Reasoning):} These require multi-step reasoning (e.g., solving word problems involving algebra or logic). Similar to the previous task, the dataset includes 10 classes with 100 examples each (90 training, 10). Data examples see \Cref{tab:math_problems}.
\end{itemize}

\begin{table*}[h!]
\centering
\caption{Description of the math problem task templates. We consider 10 different types of math word problems.}
\label{tab:math_problems}
\resizebox{0.8\textwidth}{!}{%
\begin{tabular}{l|p{9cm}}
\hline
\hline
\textbf{Math Word Problems} & \textbf{Template prompt question} \\
\hline
Remaining pizza slices & Lisa ate A slices of pizza and her brother ate B slices from a pizza that originally had C slices. How many slices of the pizza are left? Reason: Combined slices eaten = A + B. Left = C - (A + B). \\
\hline
Chaperones needed for trip & For every A students going on a field trip, there are B adults needed as chaperones. If C students are attending, how many adults are needed? Reason: Adults needed = (B * C) // A. \\
\hline
Total number after purchase & In an aquarium, there are A sharks and B dolphins. If they bought C more sharks, how many sharks would be there in total? Reason: Total sharks = A + C. \\
\hline
Total game points & Michael scored A points in the first game, B points in the second, C in the third, and D in the fourth game. What is his total points? Reason: Total points = A + B + C + D. \\
\hline
Total reading hours & Emily reads for A hours each day. How many hours does she read in total in B days? Reason: Total hours read = A * B. \\
\hline
Shirt cost after discount & A shirt costs A. There's a B-dollar off sale. How much does the shirt cost after the discount? Reason: Cost after discount = A - B. \\
\hline
Area of a garden & A rectangular garden has a length of A meters and a width of B meters. What is its area? Reason: Area = A * B. \\
\hline
Total savings & If Jake saves A each week, how much will he save after B weeks? Reason: Total savings = A * B. \\
\hline
Number of cupcake boxes & A bakery sells cupcakes in boxes of A. If they have B cupcakes, how many boxes can they fill? Reason: Boxes filled = B // A. \\
\hline
Interest earned & John invests A at an annual interest rate of B\%. How much interest will he earn after C years? Reason: Interest = (A * B * C) // 100. \\
\hline
\hline
\end{tabular}
}
\label{tab:math}
\end{table*}
% Describing the tasks for VLMs in detail
\subsubsection{VLM Influence Evaluation Tasks} For VLMs, we adapt text-to-image generation tasks from~\citep{kwon2023datainf} into image-to-text (captioning) tasks to evaluate influence:
\begin{itemize}
    \item \textbf{Style Generation:} This task involves generating captions for images in specific styles: cartoons~\citep{Norod78_cartoon_2023}, pixel art~\citep{Jainr3-pixelart_2023}, and line sketches~\citep{Zoheb_sketch-scene_2023}. Each style dataset contains 200 training and 50 test image-text pairs, totaling 600 training and 150 test samples across three styles. Data examples see \Cref{fig:example}.
    \item \textbf{Subject Generation:} Using the DreamBooth dataset~\citep{ruiz2023dreambooth}, this task generates captions for images of 30 distinct subjects (e.g., specific objects or animals). Each subject provides 3 training samples, with the remaining samples used for valuation.  Data examples see \Cref{fig:example}.
\end{itemize}
% Detailing the evaluation metrics
\subsubsection{Influential Data Detection Metrics}
We adopt two metrics from~\citep{kwon2023datainf} to assess influence:
\begin{itemize}
    \item \textbf{AUC Score:} For each test data point, we assign pseudo labels to training points (1 if the training point’s label matches the test point’s, 0 otherwise). We compute the Area Under the Curve (AUC) between data values (influence scores) and pseudo labels, averaging across all test points. A higher AUC indicates better identification of influential points.
    \item \textbf{Recall:} For each test point, we calculate the percentage of influential training points (top-ranked by influence score) that share the same class as the test point. This measures the relevance of identified influential points.
\end{itemize}
\subsubsection{Mislabeled Data Detection Data \& Metrics}
\begin{figure*}[h]
\centering
\includegraphics[width=0.95\linewidth]{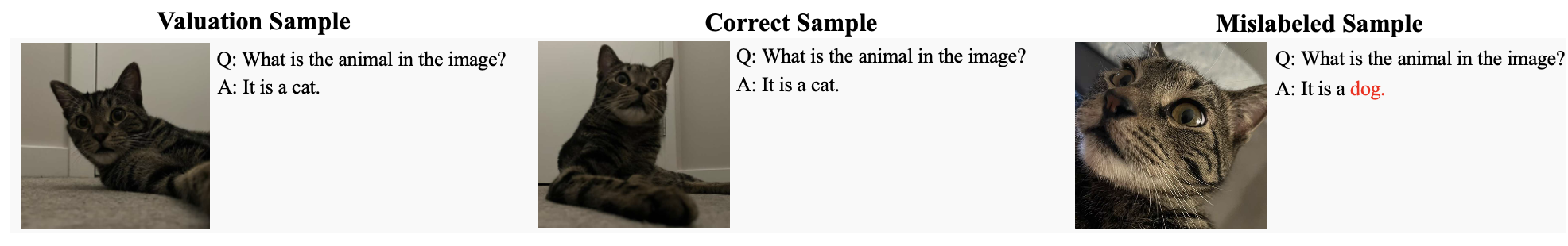} 
\caption{Description of the mislabeled data detection task. We utilize a cat versus dog classification dataset and intentionally introduce noise by randomly swapping the labels of 50\% of the data.
}
\label{fig:example_mis}
% \vspace{-3mm}
\end{figure*}
For mislabeled detection, we transform the dataset into a visual-language question answering task with the template "What is the animal in the image? It is a [label]" with demonstration\footnote{To prevent any licensing issues, the images shown are not from the original dataset; they were personally captured for demonstration purposes.} in \Cref{fig:example_mis}. We then select the first 400 images for both dogs and cats, flipping 50\% of the labels to introduce noise. For valuation, we use 200 images, with each class containing 100 images. For evaluation, we also calculate the AUC and Recall but with the pseudo labels to training points being 1 if the training point’s label matches the test point’s and it is clean data, 0 otherwise.
\subsection{Noise-Huatuo-CoT Data Example}
We construct the Noise-Huatuo-Complex-CoT dataset by randomly sampling 5,000 examples without replacement and following~\cite {wei2019eda} to inject noise into 40\% of them through random insertion or deletion of irrelevant words, as illustrated in \cref{fig:noisy}.
\begin{figure*}[h]
\centering
\includegraphics[width=1.0\linewidth]{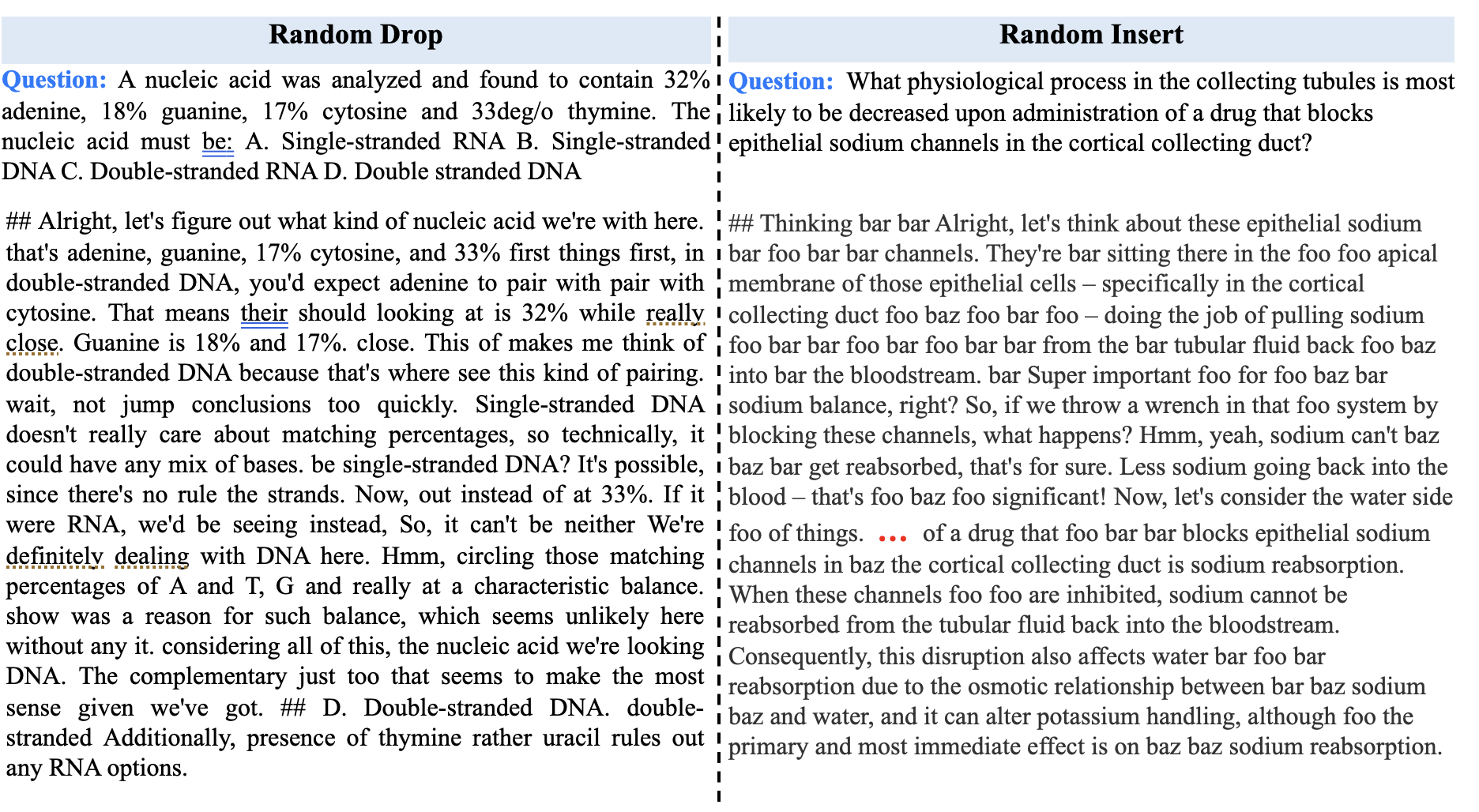} 
\caption{Examples of two types of noisy data. (Left) Random word deletion, where tokens are dropped from the reasoning, for instance, ‘Thinking’ is removed after $\#\#$. (Right) Random word insertion, where irrelevant tokens such as ‘bar,’ ‘foo,’ and ‘baz’ are injected into the reasoning. \textcolor{red}{Red} dashes means omitted reasoning.
}
\label{fig:noisy}
%\vspace{-3mm}
\end{figure*}
% Explaining baseline selection
\subsubsection{Baseline Checkpoints Selection}
For baseline methods, we select the model checkpoint with the highest test AUC, as influence function-based methods exhibit significant performance variability across training checkpoints. Notably, this variability does not correlate with validation loss, posing challenges for practical deployment. We compare \ours{} against these baselines to ensure robust evaluation.

% Adding dataset statistics in a table
\subsubsection{Dataset Statistics}
We present dataset statistics in \cref{tab:dataset_stats}
\begin{table*}[h]
\centering
\caption{Dataset statistics for LLM and VLM tasks.}
\label{tab:dataset_stats}
\resizebox{\textwidth}{!}{%
\begin{tabular}{lcc}
\toprule
\textbf{Task} & \textbf{Training Samples} & \textbf{Valuation Samples} \\
\midrule
 Sentence Transformations & 900 (90 $\times$ 10 classes) & 100 (10 $\times$ 10 classes) \\
Math Word Problems (No Reasoning) & 900 (90 $\times$ 10 classes) & 100 (10 $\times$ 10 classes) \\
Math Word Problems (With Reasoning) & 900 (90 $\times$ 10 classes) & 100 (10 $\times$ 10 classes) \\
Style Generation & 600 (200 $\times$ 3 styles) & 150 (50 $\times$ 3 styles) \\
Subject Generation & 90 (3 $\times$ 30 subjects) & Variable (1-3) per subject \\
Mislabel Detection & 800 (400 $\times$ 2 subjects 50\% noise) & 200 (100 $\times$ 2 subjects) \\
GSM8K & 7470 & 1319 \\
Noise-Huatuo-Complex-CoT & 5000 (2981 clean, 2019 noise) & 5000 (clean) \\
PMC-Reasoning (subset) & 10000 & 5000 \\
\bottomrule
\end{tabular}
}
\end{table*}

% \subsection{Usage of Large Language Model}
% In preparing this paper, we made limited use of ChatGPT to support writing and editing. Specifically, LLMs were employed for language polishing, grammar refinement, and rephrasing sentences to improve clarity and readability. Importantly, all technical content, including theoretical analysis, algorithm design, and experimental results, was conceived, implemented, and validated by the authors. LLM outputs were always critically reviewed, verified, and revised before inclusion. No LLM-generated text, figures, or tables were incorporated without careful human oversight.  
\subsection{License Clarification}\label{sec:license}

The Dreambooth images have been either taken by the authors of the paper or obtained from Unsplash\footnote{\url{https://www.unsplash.com/}}. The file located at this link\footnote{\url{https://huggingface.co/datasets/google/dreambooth/blob/main/dataset/references_and_licenses.txt}} includes a list of all reference links to the images on Unsplash, along with the photographers' attributions and the image licenses.
The sketch images are sourced from FS-COCO~\cite{fscoco}. Data attributions and image licenses can be found in the file provided at the following link\footnote{\url{https://github.com/pinakinathc/fscoco}}.

\subsection{Usage of Large Language Model}
In preparing this paper, we made limited use of ChatGPT to support writing and editing. Specifically, LLMs were employed for language polishing, grammar refinement, and rephrasing sentences to improve clarity and readability. Importantly, all technical content, including theoretical analysis, algorithm design, and experimental results, was conceived, implemented, and validated by the authors. LLM outputs were always critically reviewed, verified, and revised before inclusion. No LLM-generated text, figures, or tables were incorporated without careful human oversight.  

\clearpage
\onecolumn
\section{Proofs}
\subsection{Preliminaries}

\textbf{Auto-Regressive Pretrained LLMs and VLMs.}
We examine a pretrained large language model (LLM) or vision-language model (VLM) denoted as $\pi_\theta$, where $\theta$ represents its parameters. For a given input $\x$ — which may consist of text tokens, image patches, or a combination of both — the model defines a conditional probability distribution over an output text sequence $\y = (y_1, y_2, \ldots, y_{|\y|})$, factorized as:
\[
\pi_\theta(\y | \x) = \prod_{k=1}^{|\y|} \pi_\theta(y_k | \x, \y_{<k}),
\]
where $\y_{<k} = (y_1, \ldots, y_{k-1})$. At each step, the model predicts the next token $y_k$ conditioned on the input $\x$ and the prefix $\y_{<k}$. This auto-regressive structure underlies most modern LLMs and VLMs, which are used in tasks such as text generation~\citep{wu2025medreason}, image captioning~\citep{bai2025qwen25vltechnicalreport}, and multi-modal reasoning~\citep{achiam2023gpt}.

\noindent\textbf{Training loss of LLMs and VLMs}
To adapt a pretrained LLM or VLM to a specific domain or task, models are typically trained on a supervised dataset $\mathcal{D} = {(\x_i, \y_i)}_{i=1}^n$ of input-output pairs. Training is commonly performed using the standard teacher-forcing objective, which minimizes the negative log-likelihood of the target sequence:
\begin{align*}
    \mathcal{L}_{\mathrm{SFT}}(\theta) 
& = -\, \frac{1}{n} \sum_{i=1}^n 
\ln \pi_\theta(\y_i | \x_i)  =-\, \frac{1}{n} \sum_{i=1}^n \sum_{k=1}^{|\y_i|} 
\ln \pi_\theta(y_{i,k} | \x_i, \y_{i,<k}).
\end{align*}
This objective maximizes the likelihood that the model generates the correct output sequence conditioned on the input and the ground-truth prefix at each step. The parameters are updated using gradient descent or its variants:
\[
\theta \leftarrow \theta - \eta\, \nabla_\theta\, \mathcal{L}_{\mathrm{SFT}}(\theta),
\quad \text{with} \quad \theta_{t=0} = \theta_0,
\]
where $\eta > 0$ is the learning rate.  
Teacher forcing stabilizes fine-tuning by supplying the true prefix $\y_{<k}$ during training, enabling the model to align its predictions closely with the target data distribution in the new domain.

\subsection{Proof of \Cref{the:val}}\label{sec:proof}
In this section, we give the detailed proof of our \Cref{the:val}, we start by proving the following theorem: 
\begin{theorem}\label{the:val1}
For a data $\x_v$ and its generation $\y_v$ that await valuation, at any time $t\geq0$ of training using a training data $(\x_i,\y_i),i\in[n]$, the training data exhibits larger value to the valuation data as the following increases:
\begin{align}
&\sum_{k=1}^{|\y_v|}  \sum_{k'=1}^{|\y_{i}|} 
\alpha_{k,k'}(t) \cdot 
\left\langle \mathbf{h}_{\x_v, \y_{v,<k}}(t), \mathbf{h}_{\x_i, \y_{i,<k'}}(t) \right\rangle  + \nn \\ & \sum_{k=1}^{|\y_v|} \left\langle \w_{\y_{v,k}}(t) - \sum_{z \in \mathcal{V}} \pi_{\theta(t)}(z|\x_v) \cdot \w_z(t),
 (\w_{\y_{i,k}} - \sum_{z \in \mathcal{V}} \pi_{\theta(t)}(z|\x_v) \cdot \w_z(t)) \right\rangle \label{eq:value1}
\end{align}
\end{theorem}

\textit{Proof:}
\begin{align*}
\frac{d}{dt} \ln\pi_{\theta(t)} (\mathbf{y}_{v} |\mathbf{x}_{v})     &= \left\langle\nabla\ln\pi_{\theta(t)} (\mathbf{y}_{v} | \mathbf{x}_{v}), \frac{d}{dt} \theta(t) \right\rangle
\\&=
\left\langle\nabla\ln\pi_{\theta(t)} (\mathbf{y}_{v} | \mathbf{x}_{v}), -\eta\nabla \mathcal{L}_{D}(\theta) \right\rangle
\\&=
\left\langle\nabla\ln\pi_{\theta(t)} (\mathbf{y}_{v} | \mathbf{x}_{v}), \eta \sum_{i=1}^n\nabla\ln\pi_{\theta(t)} (y_{i} | \mathbf{x}_{i}) \right\rangle
\\
\end{align*}

As per the unconstrained features Assumption, the model's trainable parameters are 
\[\theta=\Big(\W\,,\mathbf{h}_{\x_v}\,,\,\big\{\mathbf{h}_{\x_v, \y_{v,<k}}\big\}_{k\in\{2,\ldots, |\y_v|\}}\,,\,\big\{\mathbf{h}_{\x_i, \y_{i,<k'}}\big\}_{i\in[n],k'\in\{1,\ldots,|\y_i|\}}\Big)\,.
\]

 Unfolding the gradients with respect to these parameters yields:
\begin{align}
\frac{d}{dt} \ln\pi_{\theta(t)} (\y_v | \x_v) 
    &=  \left\langle \nabla_{\W} \ln \pi_{\theta(t)} (\y_v| \x_v), 
    \sum_i^{n}\nabla_{\W} \ln \pi_{\theta(t)} (\y_i | \x_i) 
    \right\rangle \nn  \\
    &+ \sum^{|\y_v|}_{k=1} \underbrace{\left\langle \nabla_{\mathbf{h}_{\x_v,\y_{v,<k}}} \ln \pi_{\theta(t)} (\y_{v,k}| \x_v,\y_{v,<k}), 
    \sum_{i'=1}^{n_k}\nabla_{\mathbf{h}_{\x_v,\y_{v,<k}}} \ln \pi_{\theta(t)} (\y_{i',k} | \y_{v,<k}) 
    \right\rangle}_{\text{(II) Training data have the same $(\x_v,\y_{v,<k})$}} 
        % &+ \underbrace{\sum^{|\y_v|}_{k=2} n_k ||\nabla_{\mathbf{h}_{\x_v, \y_{v,<k}}}\ln\pi_{\theta(t)} (\y_{v,k} | \x_v,\y_{v,<k})||^2}_{\text{(III) training data have the same $(\x_v,\y_{v,<k}$})}
    \,. \label{eq:back in}
\end{align}
where $n_k$ is the number of training data whose input and prediction before token $k$ are the same as valuation data $(\x_v,\y_{v,<k})$.
Since we have 
\begin{align*}
\nabla_{\mathbf{W}} \ln \pi_{\theta(t)}(z|\mathbf{x}) &= \left( \mathbf{e}_z - \sum_{z' \in \mathcal{V}} \pi_{\theta(t)}(z'|\mathbf{x}) \cdot \mathbf{e}_{z'} \right) \mathbf{h}_{\mathbf{x}}^\top(t), \\
\nabla_{\mathbf{h}_{\mathbf{x}}} \ln \pi_{\theta(t)}(z|\mathbf{x}) &= \mathbf{W}_z(t) - \sum_{z' \in \mathcal{V}} \pi_{\theta(t)}(z'|\mathbf{x}) \cdot \mathbf{W}_{z'}(t).
\end{align*}
Putting this back in \eqref{eq:back in} together with a few algebra steps, yields
\begin{align}
\frac{d}{dt} \ln\pi_{\theta(t)} (\y_v | \x_v) 
&= \text{(I)} + \text{(II)} 
\end{align}
where:
\begin{align}
\text{(I)} &= \sum_{k=1}^{|\y_v|} \sum_{i=1}^{n} \sum_{k'=1}^{|\y_{i}|} 
\alpha_{k,k'}(t) \cdot 
\left\langle \mathbf{h}_{\x_v, \y_{v,<k}}(t), \mathbf{h}_{\x_i, \y_{i,<k'}}(t) \right\rangle \\
\text{(II)} &= \sum_{k=1}^{|\y_v|} \left\langle \w_{\y_{v,k}}(t) - \sum_{z \in \mathcal{V}} \pi_{\theta(t)}(z|\x_v) \cdot \w_z(t),
\sum_{i'=1}^{n_k} (\w_{\y_{i',k}} - \sum_{z \in \mathcal{V}} \pi_{\theta(t)}(z|\x_v) \cdot \w_z(t)) \right\rangle 
\end{align}
where $\alpha_{k,k'}(t) = \left\langle 
\mathbf{e}_{\y_{v,k}} - \pi_{\theta(t)} (\cdot | \x, \y_{v,<k}) ,
 \mathbf{e}_{\y_{i,k'}} - \pi_{\theta(t)} (\cdot | \x, \y_{i,<k'})
\right\rangle$. 
By taking the $i$-th sample, we can obtain \Cref{the:val1}.
\noindent We observe the following:

(1) When the training input $\x_i$ differs from the valuation input $\x_v$, its influence on the valuation target arises solely through Term (I), which captures the contribution of the token embeddings and all network parameters except the token unembedding layer.

(2) The effect of the token unembeddings is concentrated in cases where the training and valuation data share the same input $\x$ and exhibit overlapping output predictions $\y$.
\\
To eliminate this dependence on token unembeddings, we impose the following assumption:

\begin{assumption}[Distinct Input]
The training dataset satisfies that no training input $\x_i$ is identical to the valuation input $\x_v$.
\end{assumption}

\noindent
Under the Assumption 2, the contribution from token unembeddings (Term (II)) vanishes, so that the influence of the training data on the valuation data arises entirely through the shared representation features captured in Term (I).
\noindent
This assumption is mild, as training inputs typically differ from valuation inputs in practice — especially in vision-language datasets, where the input images are almost always distinct. Extending this result to cases where training examples share the same input but differ in their outputs $\y$ is straightforward: the output prefix $\y_{<k}$ can be incorporated into the input $\x$, treating each unique pair $(\x, \y_{<k})$ as a distinct input, where $k-1$ indicates the point at which the outputs begin to differ. 
Combining \Cref{the:val1} and Assumption~2 then yields \Cref{the:val}.

\end{document}